\documentclass{colt2017} 

\usepackage{amsfonts,dsfont}
\usepackage{stmaryrd}
\usepackage{caption}
\usepackage{color}
\usepackage{hyperref}
\usepackage{tikz}
\usepackage{multirow}
\usepackage{float}
\usetikzlibrary{matrix}
\usetikzlibrary{arrows}
\usetikzlibrary{decorations.pathreplacing}

\def\R{\mathbb{R}}

\def\1{\mathds{1}}

\def\card{\mathrm{card}}

\def\tr{\mathrm{tr}}

\def\span{\mathrm{span}}

\newcommand\il[1]{\langle #1 \rangle}

\title[]{Pointed subspace approach to incomplete data}

 \coltauthor{\Name{\L{}ukasz Struski} \Email{lukasz.struski@uj.edu.pl}\\
 \Name{Marek \'Smieja} \Email{marek.smieja@uj.edu.pl}\\
 \Name{Jacek Tabor} \Email{jacek.tabor@uj.edu.pl}\\
 \addr Jagiellonian University, Krak\'ow, Poland
 }

\begin{document}

\maketitle

\begin{abstract}
	Incomplete data are often represented as vectors with filled missing attributes joined with flag vectors indicating missing components. In this paper we generalize this approach and represent incomplete data as pointed affine subspaces. This allows to perform various affine transformations of data, as whitening or dimensionality reduction. We embed such generalized missing data into a vector space by mapping pointed affine subspace (generalized missing data point) to a vector containing imputed values joined with a corresponding projection matrix. Such an operation preserves the scalar product of the embedding defined for flag vectors and allows to input transformed incomplete data to typical classification methods.
\end{abstract}


\section{Introduction} \label{se:introduction}

Incomplete data analysis is an important part of data engineering and machine learning, since it appears in many practical problems. In medical diagnosis, a doctor may be unable to complete the patient examination due to the deterioration of health status or lack of patient's compliance \citep{burke1997compliance}; in object detection, the system has to recognize the shape from low resolution or corrupted images \citep{berg2005shape}; in chemistry, the complete analysis of compounds requires high financial costs \citep{stahura2004virtual}. In consequence, the understanding and the appropriate representation of such data is of great practical  importance.

A missing data is typically viewed as a pair $(x,J_x)$, where $x \in \R^N$ is a vector with missing components $J_x \subset \{1,\ldots,N\}$. In the most straightforward approach, one can fill missing attributes with some statistic, e.g. mean, taken from existing data. Although such a strategy can be partially justified when the features are missing at random, we lose the knowledge about unknown attributes\footnote{In the medical data, typically some component is missing if the state of the patient is so bad, that a given numerical procedure cannot be performed. Consequently, the knowledge that given component is missing could say a lot about the state of the patient.}. To preserve this information we usually add a flag indicating which components were missing. More precisely, we supply $x$ with a binary vector $\1_{J_x}$, in which 1 denotes absent feature while 0 means the present one.

\begin{figure}[t]
	\centering
	\begin{tikzpicture}
	\node(bit 1) at (0,4.5){\bf Input:};
	\node(bit 1) at (0,4){$(x_1,?)$};
	\node(bit 1) at (0,3.5){$(?,y_2)$};
	
	\draw[
	-triangle 90,
	line width=0.02mm,
	postaction={draw, line width=0.01cm, shorten >=0.1cm, -}
	] (3,1) -- (3,5);
	\draw[
	-triangle 90,
	line width=0.02mm,
	postaction={draw, line width=0.01cm, shorten >=0.1cm, -}
	] (1,2.7) -- (5,2.7);
	\draw[
	dashed,
	line width=0.02mm,
	postaction={draw, line width=0.01cm, shorten >=0.1cm, -}
	] (1,3.5) -- (5,3.5);
	\node(a) at (2,3.47) {\textbullet};
	\node(a) at (2.5,3.8){\tiny $(y_1,y_2) + \span\{(1,0)\}$};
	\draw[
	dashed,
	line width=0.02mm,
	postaction={draw, line width=0.01cm, shorten >=0.1cm, -}
	] (4,1) -- (4,5);
	\node(a) at (4,1.5) {\textbullet};
	\node(a) at (3.2,1.8){\tiny $(y_1,y_2) + \span\{(0,1)\}$};
	
	\node(bit 1) at (5.8,2.8){$At + b$};
	\draw[
	-triangle 90,
	line width=0.02mm,
	postaction={draw, line width=0.01cm, shorten >=0.1cm, -}
	] (5.3,2.5) -- (6.3,2.5);

	\draw[
	-triangle 90,
	line width=0.02mm,
	postaction={draw, line width=0.01cm, shorten >=0.1cm, -}
	] (8.5,1) -- (8.5,5);
	\draw[
	-triangle 90,
	line width=0.02mm,
	postaction={draw, line width=0.01cm, shorten >=0.1cm, -}
	] (6.5,2.7) -- (10.5,2.7);
	\draw[
	dashed,
	line width=0.02mm,
	postaction={draw, line width=0.01cm, shorten >=0.1cm, -}
	] (6.5,1) -- (10.5,3);
	\node(a) at (7.5,1.47) {\textbullet};
	\node(a) at (7.5,1.8){\tiny $A y + b + \span\{(0,1)\}$};
	\draw[decorate,decoration={brace}]  (6.3,2) -- node[above] {\tiny $w$} (7.1,2);
	\draw[decorate,decoration={brace}]  (7.4,2) -- node[above] {\tiny $W$} (8.6,2);
	\draw[
	dashed,
	line width=0.02mm,
	postaction={draw, line width=0.01cm, shorten >=0.1cm, -}
	] (6.5,4.5) -- (10.5,2);
	\node(a) at (8,3.53) {\textbullet};
	\node(a) at (8.5,3.8){\tiny $A x + b + \span\{(1,0)\}$};
	\draw[decorate,decoration={brace}]  (7.3,4) -- node[above] {\tiny $v$} (8.1,4);
	\draw[decorate,decoration={brace}]  (8.45,4) -- node[above] {\tiny $V$} (9.6,4);
	
	\node(bit 1) at (11.6,4.5){\bf Output:};
	\node(bit 1) at (11.6,4){$v+ V$};
	\draw[
	-triangle 90,
	line width=0.02mm,
	postaction={draw, line width=0.01cm, shorten >=0.1cm, -}
	] (12.1,4) -- (12.5,4);
	\node(bit 1) at (13.1,4){$(v,p_V)$};
	\node(bit 1) at (11.6,3.5){$w + W$};
	\draw[
	-triangle 90,
	line width=0.02mm,
	postaction={draw, line width=0.01cm, shorten >=0.1cm, -}
	] (12.1,3.5) -- (12.5,3.5);
	\node(bit 1) at (13.1,3.5){$(w,p_W)$};
	
	\end{tikzpicture}
	
	\caption{Representation of incomplete data as pointed subspaces, their affine transformation and final embedding as projections onto subspaces.}
\end{figure}
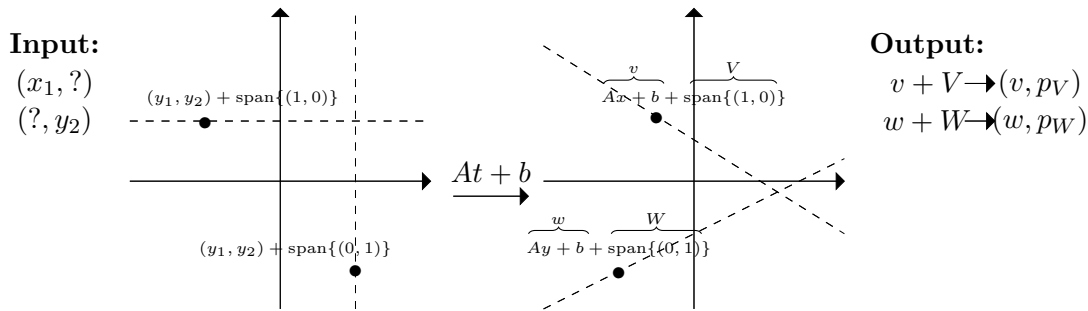

Summarizing, we perform the embedding $(x,J_x) \to (x,\1_{J_x})$ of missing points into a vector space of extended complete data. This allows us to apply typical classification tools, like SVM, with the scalar product defined by
\begin{equation} \label{eq:standard}
	\il{(x,J_x),(y,K_y)}=\il{x,y}+\il{\1_{J_x},\1_{K_y}}.
\end{equation}
In practical classification problems we usually perform various affine transformations of data, as whitening or dimensionality reduction, before training a classifier. Moreover, we may know that the data satisfy some affine constraint. It is nontrivial how to modify the flag vectors so as to keep the correspondence with such affine transformations. Thus, our main problem behind the paper can be stated as follows: \emph{How to transform the flag vectors indicating the missing components if we perform the linear (or affine) mapping of data?}

In this contribution, we show that the answer can be given by viewing the incomplete data as pointed affine subspaces, i.e. the subspace with a distinguished point called basepoint. We first observe that a pair $(x,J_x)$ can be formally associated with a pointed affine subspace of $\R^N$:
$$
x+\span(e_j)_{j \in J_x},
$$
where $(e_j)_{j =1}^N$ denotes the canonical base of $\R^N$ and $x$ is a selected basepoint. In other words, this is the set of all points which coincide with the representative $x$ on the coordinates different from $J_x$. In consequence, by a {\em generalized missing data point} in $\R^N$ we understand a pointed affine subspace $S_x = x+V$ of $\R^N$, where $x \in \R^N$ is a basepoint and $V = S_x - x$ is a linear subspace. Since the basepoint can be selected with a use of various imputation techniques, we propose to choose the most probable point of $S_x$, i.e. to project a dataset mean onto $S_x$ with respect to Mahalanobis scalar product given by the covariance of data.

Such a definition allows us to efficiently extend linear and affine operations from the standard points to missing ones, by taking the image of the subspace and the point. For example, a linear mapping $F:w \to Aw+b$, can be extended to the case of pointed subspace $x+V$ by 
$$
F(x+V)=F(x)+AV.
$$
Given an affine constraint $W$, we restrict\footnote{Observe that if such a constraint $W$ is given the augmentation of the missing components must be performed in such a way as to choose the representation in $W$, and consequently we may assume that $x \in W$.} $x+V$
by the formula $(x+V) \cap W=x+(V \cap (W-x))$.

There appears another question: how to work with such data, and in particular how to embed the generalized missing data into a vector space in such a way to respect the scalar product \eqref{eq:standard} given by the flag embedding? Our main observation shows that this can be achieved by identifying a linear subspace $V$ with an orthogonal projection $p_V: \R^N \to V$ by considering the embedding $(x,V) \to (x,p_V) \in \R^N \times \R^{N \times N}$. We show that the scalar product of embeddings coincides with \eqref{eq:standard}, i.e.
$$
\il{(x,\1_{J_x}),(y,\1_{K_y})}=\il{(x,p_{\span(e_J:j \in J_x)}),(y,p_{\span(e_k:k \in K_y)})}.
$$

The paper is organized as follows. The next section covers the related approaches to incomplete data analysis. In third section, we define the generalized missing data, present a strategy of embedding such data into a vector space and propose a new imputation method. We also define a scalar product for such embeddings and show its connections with existing flag approach. In fourth section, we illustrate our method with sample classification results.

\section{Related works} \label{se:related}

The most common approach to learning from incomplete data is known as deterministic imputation \citep{mcknight2007missing}. In this two-step procedure, the missing features are filled first, and only then a standard classifier is applied to the complete data \citep{little2014statistical}. Although the imputation-based techniques are easy to use for practitioners, they lead to the loss of information which features were missing and do not take into account the reasons of missingness. To preserve the information of missing attributes, one can use an additional vector of binary flags, which was discussed in the introduction.

The second popular group of methods aims at building a probabilistic model of incomplete data which maximizes the likelihood by applying the EM algorithm \citep{ghahramani1994supervised, schafer1997analysis}. This allows to generate the most probable values from obtained probability distribution for missing attributes (random imputation) \citep{mcknight2007missing} or to learn a decision function directly based on the distributional model. The second option was already investigated in the case of linear regression \citep{williams2005incomplete}, kernel methods \citep{smola2005kernel, williams2005analytical} or by using second order cone programming \citep{shivaswamy2006second}. One can also estimate the parameters of the probability model and the classifier jointly, which was considered in \citep{dick2008learning, liao2007quadratically}. This techniques work very well when the missing data is conditionally independent of the unobserved features given the observations, but there is no guarantee to get a reasonable estimation in more general missing not at random case.

There is also a group of methods, which does not make any assumptions about the missing data model and makes a prediction from incomplete data directly. In  \citep{chechik2008max} a modified SVM classifier is trained by scaling the margin according to observed features only. The alternative approaches to learning a linear classifier, which avoid features deletion or imputation, are presented in \citep{dekel2010learning, globerson2006nightmare}. Finally, in \citep{grangier2010feature} the embedding mapping of feature-value pairs is constructed together with a classification objective function.

In our contribution, we generalize the imputation-based techniques in such a way to preserve the information of missing features. To select a basepoint we propose to choose the most probable point form a subspace identifying a missing data point, however other imputation methods can be used as well. Constructed representation allows to apply various affine data transformations preserving classical scalar product before applying typical classification methods.



\section{Generalized incomplete data}

In this section, we introduce the subspace approach to incomplete data. First, we define a generalized missing data point, which allows to perform affine transformation of incomplete data. Then, we show how to embed generalized missing data into a vector space and select a basepoint. Finally, we define a scalar product on the embedding space.

\subsection{Incomplete data as pointed affine subspaces}

Incomplete data $X$ can be understood as a sequence of pairs $(x_i,J_i)$, where $x_i \in \R^N$ and $J_i \subset \{1,\ldots,N\}$ indicates missing coordinates of $x_i$. Therefore, we can associate a missing data point $(x,J)$ with an affine subspace $x + \span(e_j)_{j \in J}$, where $(e_j)_j$ is the canonical base of $\R^N$. Let us observe that $x+ \span(e_j)_{j \in J}$ is a set of all $N$-dimensional vectors which coincide with $x$ on the coordinates different from $J$. 

In this paper, we focus on transforming incomplete data by affine mappings. For this purpose, we generalize the above representation to arbitrary affine subspaces, or more precisely pointed affine subspaces, which do not have to be generated by canonical bases.
\begin{definition}
	A generalized missing data point is defined as a pointed affine subspace $S_x = x + V$, where $x \in \R^N$ is a basepoint and $V = S_x - x$ is a linear subspace of $\R^N$.
\end{definition}
A basepoint can be selected by filling missing attributes with a use of any imputation method, which will be discussed in the next subsection.

\begin{remark} \label{rem:pointedAff}
Observe that the notion of pointed affine subspace differs from classical affine subspace. In particular, pointed subspace depends on the selection of basepoint. In consequence, we can create two different generalized missing data points $S_y, S_z$ from the same missing data point $(x,J)$ by using different imputation methods. 
\end{remark}

First, we show that the above definition is useful for defining linear mappings on incomplete data. Let $S_x = x + V$ be a generalized missing data point and let $f:\R^N \ni w \to Aw+b$ be an affine map. We can transform a generalized missing data point $x+V$ into another missing data point by the formula:
$$
f(x+V)=\{Aw+b:w \in x+V\}.
$$
The basepoint $x$ is mapped into $Ax+b$, while the linear part of $f(x+V)$ is given by 
$$
f(x+V)-f(x)=AV.
$$
Consequently, we arrive at the definition:
\begin{definition} \label{def:affine}
	For a a generalized missing data point $S_x = x+V$ and an affine mapping $f:w \to Aw+b$ we put:
	$$
	f(x+V)=(Ax+b)+AV,
	$$
	where $Ax+b$ is a basepoint and $AV$ is a linear subspace.
\end{definition}
One can easily compute and represent $AV$, if the orthonormal base $v_1,\ldots,v_n$ of $V$ is given, namely we simply orthonormalize the sequence $Av_1,\ldots,Av_n$.

\subsection{Embedding of generalized missing data} \label{sec:imput}

The above representation is useful for understanding and performing affine transformations of incomplete data, such as whitening, dimensionality reduction or incorporating affine constraints to data. Nevertheless, typical machine learning methods require vectors or a kind of kernel (or similarity) matrix as the input. We show how to embed generalized missing data into a vector space.

A generalized missing data point $S_x = x+V$ consists of a basepoint $x \in \R^N$ which is an element of vector space and a linear subspace $V$. To represent a subspace $V$, we propose to use a matrix of orthogonal projection $p_V$ onto $V$. To get an exact form of $p_V$, let us assume that $(v_j)_{j \in J}$ is an orthonormal base of $V$. Then, the projection of $y \in \R^N$ can be calculated by
$$
p_V(y)=\sum_{j \in J} \il{y,v_j}v_j=\sum_{j \in J} v_j v_j^T y=(\sum_{j \in J} v_jv_j^T)y,
$$
which implies that 
$$
p_V=\sum_{j \in J}v_jv_J^T. 
$$

The selection of basepoint relies on filling missing attributes with some concrete values, which is commonly known as imputation. In our setting, by the imputation we denote a function $\Phi: X \to \R^N$ such that
$$
\Phi(S_x) \in S_x,
$$
for a generalized missing data $S_x$. 

In the case of classical incomplete data, missing attributes are often filled with a mean or a median calculated from existing values for a given attribute. However, these imputations cannot be easily defined in a general case, because the linear part of generalized missing data point might be an arbitrary linear subspace (not necessarily a subspace generated by a subset of canonical base). Let us observe that another popular imputation method, which fills the  missing coordinates with zeros can be defined for generalized incomplete data. This is performed by selecting a basepoint of an incomplete data point $S_x = x+V$ as the orthogonal projection of missing data $x$ onto the subspace orthogonal to $V$, i.e.:
$$
x_{V^{\perp}}=x-p_V (x)=x-\sum_{j \in J} \il{x_j,v_j}v_j,
$$ 
where $(v_j)_{j \in J}$ is an arbitrary orthonormal base of $V$. If $V$ is represented by canonical base then this is equivalent to filling missing attributes with zeros.

We propose another technique for setting missing values, which extends zero imputation method. Let us assume that $(m,\Sigma)$ are the mean and covariance matrix estimated for incomplete dataset $X$. In this method, a basepoint of $x+V$ is selected as the orthogonal projection of $m$ onto $x+V$ with respect to the Mahalanobis scalar product parametrized by $\Sigma$, i.e. 
$$
x_V^{(m,\Sigma)} = x+p_V^{\Sigma}(m-x),
$$
where $p_V^{\Sigma}$ denotes a projection matrix onto $V$ with respect to Mahalanobis scalar product given by $\Sigma$. To obtain the values for $m$ and $\Sigma$ in practice, one can use existing attributes of incomplete data for the calculation of a sample mean and a covariance matrix. Alternatively, if data satisfy missing at random assumption, then the EM algorithm can be applied to estimate the probability model describing data \citep{schafer1997analysis}. We call this technique by \emph{the most probable point imputation.} 

Summarizing, our embedding is defined as follows:
\begin{definition}
	A generalized missing data point is embedded in a vector space by
	$$
	S_x \to (x,p_V) \in \R^N \times \R^{N \times N},
	$$
	where $S_x = x+V$ and $x$ is a basepoint.
\end{definition}

\begin{example}
	To illustrate the effect of missing data imputation and transformation, let us consider the whitening operation:
	$$
	\mathrm{Whitening}(x) = \Sigma^{-1/2}(x-m),
	$$
	where $\Sigma$ is the covariance, and $m$ the mean of $X$. For a generalized missing data the above operation is defined by:
	$$
	\mathrm{Whitening}(x+V) = \Sigma^{-1/2}(x-m) + \Sigma^{-1/2} V.
	$$
	In other words, we map a basepoint in a classical way and transform a subspace $V$ into a linear subspace $\Sigma^{-1/2} V$. The illustration is given in Figure \ref{fig:whiten}.
\end{example}

\begin{example}
	In the case of high dimensional data, we sometimes reduce a dimension of input data space by applying the Principle Component Analysis, which is defined by:
	$$
	\mathrm{PCA}(x) = W^T (x-m),
	$$
	where $m$ is a mean of a dataset and $k$ columns of $W$ are the leading eigenvectors of covariance matrix $\Sigma$. This operation can be extended to the case of generalized missing data by:
	$$
	\mathrm{PCA}(x+V) = W^T (x-m) + W^T V.
	$$
	An example of the above operation is illustrated in the Figure \ref{fig:pca}.
	
\end{example}

\begin{figure*}[t]
	\centering
	\subfigure[Zero imputation.]{\includegraphics[height=0.92in]{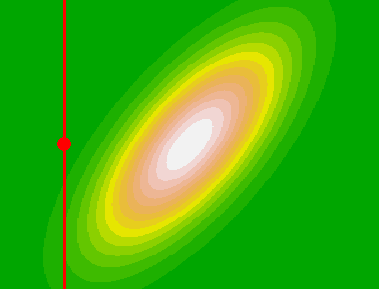}  \label{cl}} \quad
	\subfigure[Whitening for zero imputation.]{\includegraphics[height=0.92in]{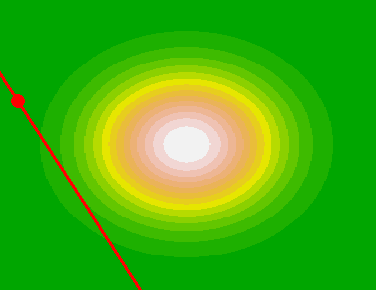} \label{cl1}}\quad
	\subfigure[Most probable point imputation.]{\includegraphics[height=0.92in]{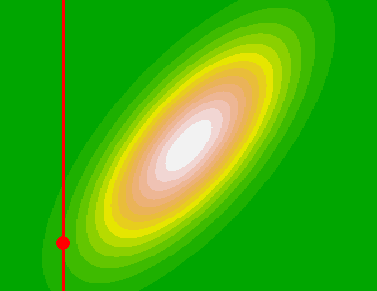} \label{pr}} \quad
	\subfigure[Whitening for most probable point.]{\includegraphics[height=0.92in]{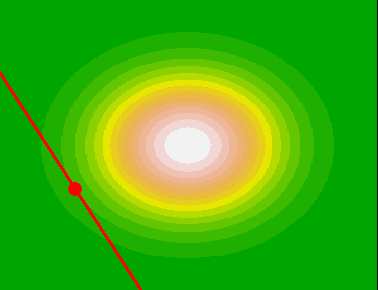} \label{pr1}}
	\caption{Whitening of data with a single element containing one missing attribute. Missing feature was filled with zero (\ref{cl}, \ref{cl1}) or most probable point imputation (\ref{pr}, \ref{pr1}).}
	\label{fig:whiten}
\end{figure*}

\begin{figure*}[t]
	\centering
	\subfigure[Image.]{\includegraphics[height=1in]{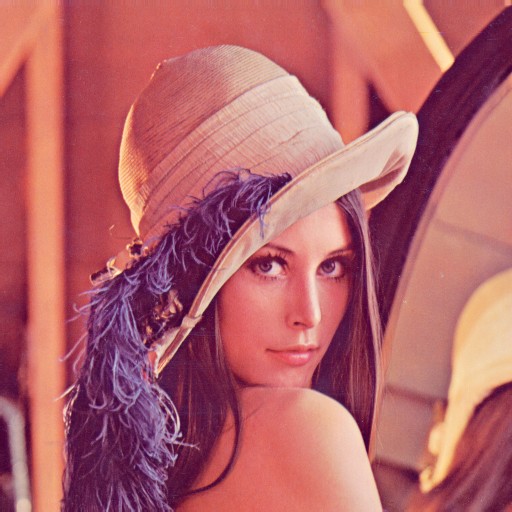} \label{im}}
	\subfigure[2D projection.]{\includegraphics[height=1in]{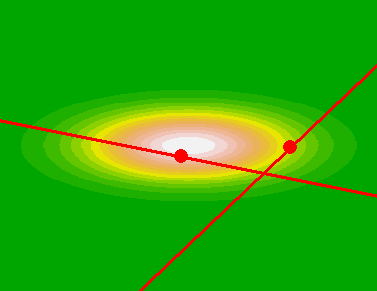} \label{pca}}
	\caption{The image \ref{im} with two missing pixels and its projection onto two principal components \ref{pca}. Image was represented by the feature vectors consisting of 8x8 blocks. Missing pixels are identified by the pointed subspaces with basepoints chosen by zero imputation strategy.}
	\label{fig:pca}
\end{figure*}

\subsection{Scalar product for SVM kernel}

To apply most of classification methods it is necessary to define a scalar product (kernel matrix) on a data space. As a natural choice, one could sum the scalar products between basepoints and embedding matrices, i.e.
\begin{equation}\label{std:product}
	\il{x+V,y+W}=\il{x,y}+\il{p_V,p_W}.
\end{equation}
However, for a data space of dimension $N$, we have $\|p_V\|^2=N$, which implies that the weight of projection can dominate the first part of \eqref{std:product} concerning basepoints. Consequently, we decided to introduce an additional parameter to allow reducing the importance of projection part:
\begin{definition}\label{product}
	Let $D \in [0,1]$ be fixed. As a scalar product between two generalized missing data points we put:
	\begin{equation} \label{eq:productt}
	\il{x+V,y+W}_D=\il{x,y}+D \il{p_V,p_W}.
	\end{equation}
\end{definition}
Let us observe that the above parametric scalar product can be implemented by taking the embedding $x+V \to (x,\sqrt{D}p_V)$ and then using formula \eqref{std:product} for a scalar product.

\begin{remark}
Observe that the value of function \eqref{eq:productt} strictly depends on the selection of basepoints, which makes it not well defined scalar product in the space of classical affine subspaces. Indeed, $x+V$ defines the same affine subspace as $x+v+V$, where $v \in V$, but such shifts may lead to different values of the right hand side of \eqref{eq:productt}. However, this is well defined scalar product in the case of pointed affine subspaces, because two different selections of basepoints give different pointed affine subspaces (see Remark \ref{rem:pointedAff}). In consequence, it might be safely used in the case of generalized missing data points considered in this paper.
\end{remark}

The following proposition shows how to calculate a scalar product between matrices defining two orthogonal projections onto linear subspaces.
\begin{proposition} \label{prop}
	Let us consider subspaces 
	$$
	V=\span(v_j: j \in J), W=\span(w_j:j \in K).
	$$
	where $v_j$ and $w_k$ are orthonormal sequences. If $p_V,p_W$ denote orthogonal projections onto $V,W$, respectively, then 
	$$
	\il{p_V,p_W}=\sum_{j \in J,k \in K} \il{v_j,w_k}^2.
	$$
\end{proposition}

\begin{proof}
	By the definition of orthogonal projections and the scalar product between matrices, we have
	\begin{equation} \label{eq:trace}
		\il{p_V,p_W}=\sum_{j \in J,k \in K}\tr((v_jv_j^T)^T(w_kw_k^T)).
	\end{equation}
	Making use of $\tr(AB)=\tr(BA)$, we get
	$$
	\tr((v_jv_j^T)^T(w_kw_k^T))=\tr(v_jv_j^Tw_kw_k^T)=
	\tr(v_j^Tw_kw_k^Tv_j)=(v_j^Tw_k) \cdot (w_k^Tv_j)=\il{v_j,w_k}^2.
	$$
	Finally, 
	$$
	\il{p_V,p_W}=\sum_{j \in J,k \in K}\il{v_j,w_k}^2.
	$$
\end{proof}

Concluding, the scalar product between embedding of two generalized missing data points given by Definition \ref{product} can be calculated as:
$$
\il{x+V,y+W}_D=\il{x,y}+D\sum_{i,j}(p_V)_{ij}(p_W)_{ij} =\il{x,y} + D\sum_{j \in J,k \in K}\il{v_j,w_k}^2,
$$
where $(v_j)_{j \in J},(w_k)_{k \in K}$ are orthonormal bases of $V,W$, respectively. The last expression can be more numerically efficient if the dimension of the subspaces (the number of missing attributes) is much smaller than the dimension of the whole space.

\begin{remark} One of typical representations of missing data $(x,J)$ relies on filling unknown attributes and supplying it with a binary flag vector $\1_J \in \R^N$, in which bit $1$ denotes coordinate belonging to $J$. This leads to the embedding of the missing data into a vector space given by  
	$$
	(x,J) \to (x,\1_J) \in \R^N \times \R^N.
	$$
	Then, the scalar product of such embedding can be defined by
	\begin{equation} \label{eq:scalar}
		\il{(x,\1_J,)(y,\1_K)}=\il{x,y}+\il{\1_J,\1_K}=\il{x,y}+\card(J \cap K).
	\end{equation}
	
	It is worth to noting that the formula \eqref{eq:scalar} coincides with a scalar product defined for generalized missing data \eqref{std:product} (for $D=1$). Indeed, if $V=\span(e_j:j \in J)$ and 
	$W=\span(e_k:k \in K)$, for  $J,K \subset \{1,\ldots,N\}$, then by Proposition \ref{prop} we have,
	$$
	\il{p_V,p_W}=\sum_{j \in J,k \in K}\il{e_j,e_k}^2=\sum_{l \in J \cap K}\il{e_l,e_l}^2=\sum_{l \in J \cap K}1=\card (J \cap K),
	$$
	which is exactly the RHS of \eqref{eq:scalar}. 
	
	Therefore, our approach generalizes and theoretically justifies the flag approach to missing data analysis. The importance of our construction lies in its generality, which in particular allows for performing typical affine transformations of data. In the case of flag representation, there is no obvious solution how to perform such mappings on flag vector.
\end{remark}

\section{Experiments} \label{se:experiments}

To illustrate our approach we applied it in SVM classification experiments, which assumed the use of whitening operation before performing a classification phase. We used examples retrieved from UCI repository combined with two strategies for attributes removal: random and structural. Finally, one real medical dataset was employed, which simulates a real process of missing features.

\begin{table}[t]
	\caption{Mean accuracies for a classification of UCI data sets with randomly missing attributes.}
	\label{tab:data1}
	\begin{center}
		\begin{tabular}{r|ccccc}
			\multicolumn{1}{c|}{\bf data} & \multicolumn{1}{c}{\bf embedding} &\multicolumn{1}{c}{\bf zero} &\multicolumn{1}{c}{\bf mean} &\multicolumn{1}{c}{\bf median} &\multicolumn{1}{c}{\bf most probable}
			\\ \hline \\
			\multirow{2}{*}{\rotatebox{90}{BC}} & no information & $0.71 \pm 0.02$ & $0.71 \pm 0.04$ & $0.73 \pm 0.04$ & $0.76 \pm 0.02$ \\
			& subspace & $0.73 \pm 0.03$ & $0.73 \pm 0.05$ & $0.74 \pm 0.04$ & $0.77 \pm 0.02$ \\ \hline \\
			
			\multirow{2}{*}{\rotatebox{90}{IS}} & no information & $0.63 \pm 0.02$ & $0.67 \pm 0.02$ & $0.67 \pm 0.02$ & $0.67 \pm 0.03$ \\
			& subspace & $0.65 \pm 0.02$ & $0.67 \pm 0.02$ & $0.67 \pm 0.03$ & $0.67 \pm 0.02$ \\ \hline \\
			
			\multirow{2}{*}{\rotatebox{90}{Y}} & no information & $0.49 \pm 0.02$ & $0.52 \pm 0.01$ & $0.52 \pm 0.01$ & $0.52 \pm 0.01$ \\
			& subspace & $0.5 \pm 0.02$ & $0.52 \pm 0.01$ & $0.52 \pm 0.02$ & $0.53 \pm 0.01$ 
		\end{tabular}
	\end{center}	
\end{table}

For all cases, the following procedure was applied. First, we set missing features with a use of one of four strategies mentioned in the paper:
\begin{enumerate}
\item {\bf Mean:} average value of the feature over training set. 
\item {\bf Median:} median of the feature over training set. 
\item {\bf Zero imputation:} missing features were filled with zeros. 
\item {\bf Most probable imputation:} it was described in section \ref{sec:imput}.
\end{enumerate}
For a simplicity the mean and covariance matrix were estimated from a training set with a use of {\bf norm} R package\footnote{Since the use of EM method implemented in {\bf norm} is justified in missing at random case, then one could also estimated a mean and covariance based on existing attributes.}. 

Next, we performed a whitening of dataset (making use of the parameters returned by {\bf norm}) based on two approaches:
\begin{enumerate}
\item {\bf No information:} Feature vectors with imputed missing attributes were whiten. 
\item {\bf Subspace:} Feature vectors with imputed values were joined with corresponding projection matrices and then the entire vectors were whiten according to the Definition \ref{def:affine}.
\end{enumerate}
The above scenarios represent classical imputation and our pointed affine subspace approach. We would like to investigate how the information preserved in the subspace influences the classification results.

Finally, we calculated the scalar products (kernel matrices) for such representations of data and  trained SVM classifier implemented in {\bf libsvm} \citep{chang2011libsvm}. Missing features of test set instances were filled and transformed  based on a training set only. 

All experiments assumed double 5-fold cross validation. More precisely, for every division into train and test sets, the required hyperparameters were tuned using inner 5-fold cross validation applied on training set. The combination of parameters maximizing mean accuracy score (on validation set) was used to learn a final classifier on a entire training set, while the performance was evaluated on a testing set that was not used during training. The accuracy was averaged over all 5 trails. 
We learned a standard margin parameter $C$ as well as a parameter $D$ in the formula of scalar product for subspace embedding. We performed a grid search in the following ranges: $C=\{10^k\, :\ k = -2,-1,0,1\}$ and $D=\{\frac{1}{2^k}\, :\ k = 0,1,\ldots,10\}$.

\subsection{UCI datasets}

We used three UCI datasets (for datasets with more than two classes we selected two the most numerous classes):  breast cancer (BC), ionosphere (IS) and yeast (Y) \citep{Asuncion+Newman:2007}. In the first case, we randomly removed $90\%$ of features. In the second option, we defined a structural process of attributes removal. More precisely, we drawn $N$ points $x_1,\ldots,x_N$ of a dataset $X \subset \R^N$. Then, for every $x \in X$ we removed its $i$-th attribute with a probability $\exp(-t\|x - x_i\|_{\Sigma}))$, where $\| x \|_{\Sigma}$ denotes the Mahalanobis norm of $x$ with respect to $\Sigma$ and $t > 0$ was chosen to remove approximately $90\%$ of attributes.

The results presented in Table \ref{tab:data1} show that there is no benefit from identifying absent attributes when the features were missing completely at random. One can observe that most probable point imputation usually provided the highest accuracy among the imputation strategies.

In the case of structurally missing features, Table \ref{tab:data2}, the proposed subspace approach gave better classification results for all datasets and for all cases of imputations. Moreover, the most probable point imputation outperformed other strategies of filling missing coordinates on two out of three datasets.

\begin{table}[t]
	\caption{Mean accuracies for a classification of UCI data sets with structural attribute absence.}
	\label{tab:data2}
	\begin{center}
		\begin{tabular}{r|ccccc}
			\multicolumn{1}{c|}{\bf data} & \multicolumn{1}{c}{\bf embedding} &\multicolumn{1}{c}{\bf zero} &\multicolumn{1}{c}{\bf mean} &\multicolumn{1}{c}{\bf median} &\multicolumn{1}{c}{\bf most probable}
			\\ \hline \\
			\multirow{2}{*}{\rotatebox{90}{BC}} & no information & $0.74 \pm 0.03$ & $0.73 \pm 0.03$ & $0.73 \pm 0.02$ & $0.76 \pm 0.02$ \\
			& subspace & $0.76 \pm 0.03$ & $0.76 \pm 0.02$ & $0.76 \pm 0.03$ & $0.78 \pm 0.02$ \\ \hline \\
			
			\multirow{2}{*}{\rotatebox{90}{IS}} & no information & $0.66 \pm 0.02$ & $0.67 \pm 0.03$ & $0.69 \pm 0.03$ & $0.69 \pm 0.03$ \\
			& subspace & $0.71 \pm 0.03$ & $0.70 \pm 0.04$ & $0.71 \pm 0.02$ & $0.72 \pm 0.02$ \\ \hline \\
			
			\multirow{2}{*}{\rotatebox{90}{Y}} & no information & $0.61 \pm 0.03$ & $0.52 \pm 0.01$ & $0.52 \pm 0.01$ & $0.52 \pm 0.01$ \\
			& subspace & $0.62 \pm 0.04$ & $0.56 \pm 0.02$ & $0.59 \pm 0.02$ & $0.56 \pm 0.02$ 
		\end{tabular}
	\end{center}	
\end{table}

\subsection{Medical data}

In this application we considered a real angiological dataset acquired from Jagiellonian Center of Experimental Therapeutic containing patients' examinations, \url{http://jcet.eu/new_en/}. The goal was to find patients with atherosclerosis. Innovative medical tests are very expensive, time-consuming and in some cases they cannot be successfully completed due to the patient's condition. In consequence, research database contains many empty cells, which is the effect of purely structural process. Since some of parameters are discrete as well as real valued numbers presented in different scales, then a whitening of data is a natural preprocessing step.

The results illustrated in Table \ref{tab:data3} partially confirm the hypothesis suggested in previous experiment. Indeed, the use of proposed subspace embedding, gave higher accuracy for all imputation strategies, but the benefit from its application was not significant. It is difficult to decide which imputation strategy was optimal because all of them provided comparable results.

\section{Conclusion}

The paper generalized the existing approach of identifying missing attributes with binary flags. To enable appropriate affine transformations of data, we represented incomplete data as pointed affine subspaces and embedded them into a vector space by linking a pointed subspace with a basepoint joined with a corresponding projection matrix. In the same spirit we proposed to select a basepoint as the most probable point from a subspace, which extends the well-known zero imputation strategy. Such a combination provided the best performance in conducted classification experiments in most cases.

\begin{table}[t]
	\caption{Mean accuracies for a classification of medical data.}
	\label{tab:data3}
	\begin{center}
		\begin{tabular}{ccccc}
			&\multicolumn{1}{c}{\bf zero} &\multicolumn{1}{c}{\bf mean} &\multicolumn{1}{c}{\bf median} &\multicolumn{1}{c}{\bf most probable}
			\\ \hline \\
			no information & $0.82 \pm 0.03$ & $0.81 \pm 0.02$ & $0.81 \pm 0.03$ & $0.81 \pm 0.02$ \\
			subspace & $0.82 \pm 0.01$ & $0.83 \pm 0.02$ & $0.83 \pm 0.02$ & $0.83 \pm 0.01$ \\
		\end{tabular}
	\end{center}	
\end{table}



\bibliography{./ref}

%
%
%
%

\end{document}